\documentclass[english,12pt,draftcls,onecolumn]{IEEEtran}
\usepackage{color,graphicx,amsmath,amssymb,amsthm,epsfig,mathrsfs,cite,bm}

\def\figref#1{Fig.\,\ref{#1}}%

\def\F{\mathcal{F}}
\def\e{\mathcal{E}}

\def\hb{\hat{\beta}_l}

\def\one{\mathcal{I}}

\def\E{\mathbb{E}}

\def\p{p}

\def\sinr{\mathtt{SINR}}

{}
  \newtheorem{thm}{Theorem}

\usepackage{amssymb}
\usepackage{float}
\usepackage{tikz}
\usetikzlibrary{trees}
\usepackage{algorithm}
\usepackage{algpseudocode}
\usepackage{lipsum}
\usepackage[lofdepth,lotdepth]{subfig}
\usepackage{enumerate}
\title{Reinforcement learning techniques for Outer Loop Link Adaptation in 4G/5G systems}

\ifCLASSOPTIONtwocolumn
\author{\hspace{-1cm}Saishankar K.P. \hspace{5cm} Sheetal Kalyani  \\ 
 Qualcomm India Private Limited, \hspace{2cm} Department of Electrical Engineering \\ \hspace{0.5cm} Hyderabad, India 500081  \hspace{3cm} Indian Institute of Technology, Madras,\\ \hspace{-1cm}kpsaishankar@gmail.com \hspace{4cm}
  Chennai, India 600036 \\ \hspace{7cm}
   skalyani@ee.iitm.ac.in
 }
\else
\author{\hspace{-1cm}Saishankar K.P. \hspace{5cm} Sheetal Kalyani  \\ 
 Qualcomm India Private Limited, \hspace{2cm} Department of Electrical Engineering \\ \hspace{0.5cm} Hyderabad, India 500081  \hspace{2cm} Indian Institute of Technology, Madras,\\ \hspace{-1cm}kpsaishankar@gmail.com \hspace{4cm}
  Chennai, India 600036 \\ \hspace{7cm}
   skalyani@ee.iitm.ac.in
  }
\fi

\begin{document}
 
\maketitle

\begin{abstract}
Wireless systems perform rate adaptation to transmit at highest possible instantaneous rates. Rate adaptation has been increasingly granular over generations of wireless systems. The base-station uses SINR and packet decode feedback called acknowledgement/no acknowledgement (ACK/NACK) to perform rate adaptation. SINR is used for rate anchoring called inner look adaptation and ACK/NACK is used for fine offset adjustments called Outer Loop Link Adaptation (OLLA). We cast the OLLA as a reinforcement learning problem of the class of Multi-Armed Bandits (MAB) where the different offset values are the arms of the bandit. In OLLA, as the offset values increase, the probability of packet error also increase, and every user equipment (UE) has a desired Block Error Rate (BLER) to meet certain Quality of Service (QoS) requirements. For this MAB we propose a binary search based algorithm which achieves a Probably Approximately Correct (PAC) solution making use of bounds from large deviation theory and confidence bounds. In addition to this we also discuss how a Thompson sampling or UCB based method will not help us meet the target objectives. Finally, simulation results are provided on an LTE system simulator and thereby prove the efficacy of our proposed algorithm.
\end{abstract}
\section{Introduction}
Typically, applications for machine learning/reinforcement learning based algorithms involve processing of large quantities of data. As a result, these algorithms have not earned much traction in the field of mobile communications outside the areas of spectrum sensing and cognitive radio. In this paper, we would like to present a novel application for learning based algorithms in rate adaptation in 4G and even future 5G systems. As we move from 2G to 4G we notice that there is finer adaptation of transmission rates over a greater range (maximum transmission rate as well as the granularity of the rates used have increased across generations) to tide over the variations inherent to the wireless channel. This trend in rate adaptation is expected to continue for 5G as well with 5G networks expected to support a much higher capacity. Rate adaptation has two parts, inner loop adaptation where, the SINR measured by the user is used as an anchor to determine the transmission rate. This transmission rate is fed back and OLLA is used at the base-station to make appropriate corrections to this transmission rate.{\bf This OLLA can be formulated as exploring offsets around the transmission rate and then using the ACK/NACK information from each offset as reward to quantify the performance of the different offsets and make an optimal correction to the rate.} This is the objective we achieve in our work by using a classical reinforcement learning technique called the Multi-armed Bandit.

 Rate adaptation using various Transmission Rates (TR) has played a crucial role in exploiting the instantaneous channel capacity by communicating at a rate (bits per channel use) that is optimally suited to the current channel and interference conditions. The advancements in practical wireless systems has also been characterized by increase in the number of possible transmission rates supported by the system and this trend is likely to be unchanged. For example, EDGE supported 2 rates \cite{rappaport1996wireless} and LTE supports 28 rates \cite{sesia2009lte}.

 A rate metric which reflects the channel capacity is computed at the UE and then it is quantized and fed back to the BS. For example, LTE supports 4 bit quantization wherein the quantized feedback (fed back to the eNodeB called Channel Quality Indicator (CQI)) is a number between 0 and 15 \cite{sesia2009lte}. This 4 bit CQI value ($C^u$) is then mapped to a 5 bit Modulation and Coding Scheme MCS ($M^u$) at the eNodeB, which takes 28 possible states. 
We can denote the mapping as $M^u=f(C^u)$. The CQI to MCS mapping in LTE depends upon the CQI fed back and the quantity of resources allocated. The MCS in an LTE system denotes the TR. In this work, we are interested in adding an offset to our TR and use the resulting rate for transmission $M^u=f(C^u) + l $ and we would like to find out the value of $l$ that is ``optimal". { This correction factor is necessary for the following reasons:
\begin{enumerate}
\item The $\sinr$ estimation at the UE is done using pilot symbols which are limited in number and accurately estimating the rate metric may not be possible. 
\item The UE may sometimes report back a biased conservative/optimistic rate estimate leading to a very low or high Block Error Rate (BLER) \footnote{A conservative rate is one were a very low rate is transmitted and the probability of successful transmission is high and an optimistic rate means a high rate with low decoding success.}. Moreover, different users may have different target BLER criteria depending upon the required quality of service (QoS).
\item The CQI estimation algorithm used at the UE is not known to BS and an OLLA algorithm is necessary at the BS to correct for any discrepancies in the reported CQI.
\item The reported rate could become outdated by the time the eNodeB uses the information.
\end{enumerate}
It is important to note that all the above artefacts will be UE specific. Therefore, a fixed rate correction at the base-station will not suffice. Consequently, the BS needs to run an OLLA algorithm specific to each UE. } These reasons are discussed in previous papers such as \cite{kari,poor2013,dur2015,pull2015} in detail. The work in \cite{kari} is an $\sinr$ back-off technique based on the ACK/NACK history and can be deployed only at the UE because the continuous valued $\sinr$ is known only at the UE. Note here that, ACK stands for acknowledgement that a transmitted packet has been decoded successfully and NACK indicates a failed decoding. The work in \cite{poor2013} also can be implemented at the UE only because it involves changing the $\sinr$ thresholds for selecting the different rates. No matter what link adaptation is used at UE, the BS would also like to have a link adaptation algorithm because the UE algorithms are proprietary and the BS has no way of knowing the link adaptation algorithm at the UE if any. The BS will always have a quantized rate/ $\sinr$ metric and a knowledge of whether a packet was successfully received and this can be used to calculate the BLER as well.  

The work in \cite{dur2015} performs OLLA at the BS by converting the CQI to $\sinr$ and then performing a bias correction on the resultant $\sinr$. It is claimed that by averaging the initial offset across users the algorithm can converge to the true offset value within a short time. The drawback of this scheme is that averaging across different UEs' offset may be a problem because the UEs are proprietary i.e., the CQI computing and reports change depending upon the UE vendor and model. Hence different UEs may have different ways of converting $\sinr$ to CQI for reporting. Hence using an average offset value across users is equivalent to assuming that all UEs see similar reporting errors. Since different UEs will have very different CQI reporting algorithms as well as different $\sinr$ levels, such a scenario is unlikely. 

The work in \cite{pull2015} exclusively discusses the impact of outdated feedback on rate adaptation. It attempts to predict the CQI evolution over time based on past observations. It assumes that, though the feedback is outdated, the feedback value itself is accurate. It is a purely prediction based scheme and doesn not exploit the available ACK/NACK data. Though \cite{pull2015} is able to handle outdated feedback, it cannot work with erroneous or biased CQI values. On the other hand, we can handle erroneous CQI values and even outdated CQI by modelling is as the true CQI plus an error term. Therefore, the scheme proposed here can handle erroneous and outdated CQI since it exploits the available ACK/NACK information in its OLLA formulation efficiently. 

 In this work, we are interested in finding an offset to the reported rate so that a target BLER can be achieved. We formulate the problem of finding the offset that achieves the target BLER as a Multi-Armed Bandit (MAB) problem. The proposed technique treats each UE differently and a unique offset is computed for each UE. 

The offset typically depends upon two major criteria and they are:
\begin{enumerate}
\item There is a maximum target BLER and the actual BLER should be around this target BLER.
\item While maintaining the target BLER, maximum possible throughput has to be achieved.
\end{enumerate}
If reaching the target BLER was the only criterion, one could easily achieve it by using a large negative $l$ in the mapping $M^u=f(C^u)+ l $. However, this affects the throughput adversely. Therefore we search for a value of $l$ in the space $-L \leq l \leq L$ resulting in $2L+1$ possible values of $l$. Higher the value of $l$, greater the transmission rate but at the same time, the BLER may be higher than the target value. Therefore we would like to find the highest value of $l$ such that the achieved BLER is close to the target BLER. { We make an assumption that a common offset can be utilized for all the MCS values associated with a particular user. This simplifying assumption enables us in formulating the above problem as an MAB problem and later it will be shown through simulations that, despite this simplifying assumption our initial objectives are met.}  

{ Let us summarize our objectives, we have a set of offsets which are the possible actions. We have a target performance in mind and we would like to sample the various offsets so as to achieve the target BLER. Every time an offset is sampled we obtain an ACK/NACK which should drive our next sampling decision and eventually ensure that the target BLER is met. This is a classical reinforcement learning problem which shall continue to be relevant even in future wireless mobile communication systems. }
We endeavour to exploit the similarity of the problem at hand with the classical MAB problem studied in references \cite{schlag1998imitate,whittle1980multi,vakili2013}. We believe that we are the first to come up with this approach for the OLLA problem. An MAB problem is a game where, there are multiple play options available and they are treated as different arms and each option (arm) has an unknown expected reward and one explores to find the the best strategy so play/pull the different arms for maximizing the reward over time or for finding an arm that matches the objectives of the player. MAB and its properties have been extensively studied and used successfully for other applications in communications as in \cite{yang2015,qian2012,bag2015,yi2014,keq2010,oks2015}. In \cite{yang2015} this approach has been used to dynamically select one channel among multiple available channels for transmitting data  such that throughput is maximized. In \cite{qian2012}, an MAB formulation is used to develop a frequency hopping strategy to prevent jamming attacks. Similarly, variants of the MAB such as restless MAB and decentralized MAB have been used for cognitive radio and cognitive compressive sensing applications in \cite{bag2015,yi2014,keq2010,oks2015}.  In the problem considered here, similar to \cite{yang2015,qian2012} at any given point only one arm can be pulled and this corresponds to using a given offset value for a transmission i.e., we can use a single transmission rate which is a function of the MCS and the offset. We would like to explore the BLER probability of using different offsets by exploring the ACK/NACK behaviour for each offset and eventually select that arm which is closest to the desired BLER. 

 The above MAB based papers suggest a strategy to sequentially pull the arms in order to maximize the reward obtained over time. This is usually done by assigning weights to each arms and playing each arm depending on the weight assigned to it. These weights are assigned based on the rewards obtained from playing each arm and this strategy is called indexing strategy. Alternatively, the work in \cite{even2002pac} provides an approach called the $(\epsilon,\delta)$ approach where there is an exploration phase and the objective of the exploration phase is to exploratively play the given set of arms so as to finally obtain an arm that is $\epsilon$ close to the optimal arm with probability $1-\delta$. In typical reward maximization problems, the strategy employed by the likes of \cite{yang2015,qian2012} works well. In the current problem our objective is not to simply maximize a reward but to achieve a target BLER. This is a scenario where one knows the exact value of the optimal reward and the objective is to find an arm close to that value. The indexing strategy is used when we would like to maximize a reward and the maximum possible reward is unknown. {For this purpose we could artificially have a cost associated with each arm of the form $f(|\beta_i-\beta|)$ where $\beta_i$ is the success probability of arm $i$ and $\beta$ is the desired success probability and $f(.)$ is a convex function. In this formulation however, the choice of $f(.)$ is open ended and there is no intuitive choice for the convex function based on the rate optimization problem at hand. Moreover the target probability $\beta$ need not be achieved on a per UE basis, it is a system wide metric. Thus even if individual UEs are slightly away from the target BLER as long as the system wide BLER is maintained, we can achieve the desired network capacity.  Since we are interested in a) achieving a known target BLER as opposed to a scheme that maximizes reward and b) ensuring a system wide target BLER rather than per user control, we decided to follow the $(\epsilon,\delta)$ approach. This helps us is proposing an algorithm tailored to the MAB based OLLA problem. Apart from this, we also discuss Thompson sampling and Upper Confidence Bound based algorithms and briefly explain why they do not help us reach the target with the same intuitive elegance of the PAC approach developed in this paper.}

Specific to our problem, there are $2L+1$ arms corresponding to different back-off values. We would like to achieve a target BLER $\alpha$ while maximizing the value of $l$ used in the mapping $M^u=f(C^u)+ l$. We would like to maximize the value of $l$, because higher the value of $l$, higher the transmission rate. Let us denote by $\alpha_l$, the BLER when the $l$-th arm is used. It is apparent that the value of $\alpha_l$ increases as $l$ increases. Now a question may arise that rather than maximizing $l$ such that $(\alpha_l \leq \alpha)$, one could use the $l$ that maximizes the expected good-put $G=\alpha_l r_{u,l}$ where $r_{u,l}$, is the rate $M^u=f(C^u)+ l$ under the condition that $(\alpha_l \leq \alpha)$.This might be a good approach for a single-shot transmission. However, when we deal with systems where re-transmission is allowed in the case of an erroneous transmission, we are only required to transmit at the highest possible rate which meets a target BLER criterion. It is to be noted that as $l$ increases, a higher TR and consequently higher rate is transmitted, and this results in a reduced probability of successful decoding. Thus, we have $\alpha_{i+1} \geq \alpha_i$ consequently making the arms of the bandit dependent on each other. From previous discussions, it is apparent that our MAB problem has dependent arms similar to \cite{jia2011unimodal,pandey2007multi} and our objective is to solve for achieving the target BLER.

 If the rewards from one arm can give information about other arms, such an MAB is called dependent MAB.  In \cite{jia2011unimodal,pandey2007multi} it is clearly mentioned that utilizing the dependent arms will save exploration cost. We identify and exploit the unique structure of the dependence present in the OLLA problem to propose a novel exploration algorithm that reduces the exploration time. We would like to find the value of $l$ such that $|\alpha_l-\alpha|$ is minimum where $\alpha$ is the target BLER. To achieve this we would ideally try pulling the different arms of the bandit, till a particular arm is found to achieve the target BLER. Since we do have a well defined ordering in the set of arms i.e., $\alpha_{i+1} \geq \alpha_i$ , we develop a unique mechanism for pulling the arms and theoretically prove it to be $\epsilon$ close to the optimal solution with probability $(1-\delta)$. We use tail-bounds for binomial distributions in order to propose the mechanism by which we pull the arms.

{ Summarizing our main contributions:
\begin{itemize}
\item We are the first to cast the OLLA problem at the BS into an MAB frame work.
\item We also identify and argue that the PAC solution to the MAB problem would be intuitive to formulate and the nature of the problem considered is such that the PAC solutions is more ideally suited to achieving our twin objectives of a) maintaining high throughput and b) keeping the BLER within desired levels.
\item We exploit the dependent arm structure unique to the OLLA problem to achieve significant reduction of exploration.
\item Furthermore since we have a target reward, we are able to apply certain bounds from large deviations theory to estimate how close a given arm is, to the desired optimal arm. These bounds are novel from a general MAB context as well.
\end{itemize}}
\subsection{Brief Outline}
We start by giving a system model and then introducing the details of casting the OLLA problem as an MAB problem. Then we explain how the ordered arms of the bandit can be exploited for a reduction in exploration complexity. Then we provide some more improvements to the basic algorithm proposed by us and finally demonstrate through simulations, the efficacy of the proposed technique. We compare the proposed technique with a method in \cite{ariyanthe} which observes the occurrence of error clusters and reduces the transmission rate when there is an error cluster. In  \cite{ariyanthe} clusters of block errors are used to add a correction factor to the TR and this scheme is applied on an LTE network at the eNodeB. To this we also add a step of increasing the transmission rate when a very large error free cluster occurs. We finally provide results which show that the scheme proposed by us is able to control the BLER at the desired level without compromising on the transmission rate, thus increasing the overall throughput.

\section {System Model}\label{symmo}
We consider a mobile communication system which supports multiple transmission rates (TR) denoted by indices $\mathbb{M}=\{M_0,M_1 \hdots M_m\}$. At any given instant, the BS has to choose the most suitable rate among $\mathbb{M}=\{M_0,M_1 \hdots M_m\}$ for each UE being scheduled for downlink data reception. To aid the BS in deciding a suitable TR, the UE feeds back a metric called the CQI which is a function of the channel conditions seen by the UE. The system under consideration also supports multiple transmissions in case the transmitted packet is not decoded successfully the first time. Such a system needs two types of feedback by the UE -a) CQI feedback : The CQI which reflects the current channel conditions b)ACK/NACK feedback : If the transmission has been received successfully the UE sends a packet acknowledgement (ACK) to the eNodeB else it sends a negative acknowledgement (NACK).

The CQI feedback comes from a set $\mathbb{C}=\{C_0,C_1 \hdots C_c\}$. The CQI is used at the BS to compute a TR $\in \mathbb{M}$ and this process is called link adaptation.This link adaptation is denoted as follows:
\begin{align}
M^u=f(C^u).
\end{align}
 Then the ACK/NACK feedback can be used to fine tune the link adaptation and is called the Outer Loop Link Adaptation. If the ACK feedback history is denoted by $\{\one_{ACK_i}\} \forall i \in (1,2, \hdots N)$ the OLLA is denoted by
\begin{align}
M^u=f(C^u)+l(\one_{ACK_i}),\label{eq:offg}
\end{align}
where $l(\one_{ACK_i})$ is an offset added to transmission rate index and is a function of the previous ACK/NACK feedback history.

 The ACK/NACK can used to calculate BLER by empirical averaging. 
\begin{align}
\hat{\alpha}_l=\sum_{i=1}^N \frac{\one_{NACK_i}}{N}, \\
\hat{\beta}_l=\sum_{i=1}^N \frac{\one_{ACK_i}}{N}, \label{eq:alpestin}
\end{align}
The BS requests CQI feedback from each user $C^u$ and maps on to an TR $M^u$. As discussed in the introduction, this TR may be too optimistic or conservative and we would like to correct this by using the mapping:
\begin{align}
M^u=f(C^u)+l,\label{eq:off}
\end{align}
where $-L\leq l \leq L$. In this work, we would like to achieve a BLER of $\alpha=(1-\beta)$ while maximizing the throughput. 
From these estimated packet loss probabilities, we are interested in finding the value of $l$ for which $\beta_l$ is close to a target $\beta$. We assume that there is at least one value of $l \in {-L,..,L}$ for which $\beta-\epsilon \leq \beta_l \leq \beta+\epsilon$.  In the next section we set this problem as an MAB and propose solving mechanisms.

\section {Outer Loop Link Adaptation as an MAB and an Approximate Solution}
At each feedback instant the BS receives a CQI feedback $C^u$ from user $u$. From this the BS can perform one of the following actions:
\begin{itemize}
\item $M^u=f(C^u)-L$,
\vdots
\item $M^u=f(C^u)-1$,
\item $M^u=f(C^u)$,
\item $M^u=f(C^u)+1$,
\vdots
\item $M^u=f(C^u)+L$
\end{itemize}
and use the resulting $M^u$ for transmission. These actions are the different arms and are indexed from $-L$ to $L$. As we increase $L$ we are searching over a greater range. We are interested in finding the value of $l$ such that the probability of successful transmission is $\beta=1-\alpha$. We assume that the $l$-th arm has this probability to be $\beta_l=1-\alpha_l$, and 
\begin{align}
\beta_i \geq \beta_{i+1} \:\:\:\: \forall i. \label{eq:struct}
\end{align}
{ Here we have made $\beta_l$ independent of $M^u$ and $C^u$ while in reality there maybe some dependence. This assumption is key in not only formulating the solution to this problem but also formulating the problem as a simple MAB with $2L+1$ offsets. While this assumption simplifies the problem to a tractable form, our results show that we still reach our target BLER without compromising on throughput and are able to do so with excellent control. In our simulation model we do not force the $\beta_l$ to be independent of $M^u$ and $C^u$. }

One would like to find the arm $l_{opt}$,
\begin{align}
l_{opt}= \max \:\: l \text{ subject to } \beta_l \geq \beta,
\end{align}
in order to obtain the maximum throughput and still satisfy the BLER constraint.
If $\beta_l$s are known this is a deterministic optimization problem. However the $\beta_l$s are success probabilities associated with the different offsets, which are unknown and must be estimated. This process of estimation is performed by exploring the different arms and using the ACK/NACK to empirically determine the success probabilities $\beta_l$s. As we pull an arm more number of times, the confidence of the estimate $\hat{\beta_l}$ increases.  These estimates would converge to the true probabilities only asymptotically. Therefore, we provide techniques for approximate solutions which are $\epsilon$ close to the optimal solution with probability $1-\delta$. 

A general MAB problem is one where, there are $L$ arms and each of this arm has a reward distribution $f_l(r)$ where $r$ is the reward obtained. The work in \cite{even2002pac} considers an MAB where the reward is binary with different arms having different reward probabilities. To reiterate, in a typical MAB problem when a reward is obtained, the quantum of reward is not dependent on the arm however the probability with which the reward is obtained changes across arms. In our problem though, as the offset increases the probability of reward decreases while the quantum of reward increases. Therefore our MAB is a dependent arm MAB as studied in \cite{pandey2007multi}. Furthermore, we have a target BLER/ success rate to attain rather than mere maximization of the reward. Since our target is to achieve a known BLER, we use the $(\epsilon,\delta)$ approach of selecting the offset that is at least $\epsilon$ close to the known target probability with probability $(1-\delta)$. We start with a general algorithm that achieves the $(\epsilon,\delta)$ solution. 

 The PAC solution from \cite{even2002pac} is given as follows:
\begin{algorithm}[h]
 \caption{PAC Algorithm by Median Elimination} \label{pac}
\begin{algorithmic}[1]
\State The set of $2L+1$ arms, $\epsilon>0,\delta>0$, $\delta_1=\frac{\delta}{2}$, $\epsilon_1=\frac{\epsilon}{4}$
\State Sample all remaining arms sequentially , until each arm is sampled $\frac{4}{\epsilon_1^2}log(\frac{3}{\delta_1})$ times.
\State Let $\hat{\beta_l}$ be the reward.
\State Compute median $\beta_l$ and eliminate all arms with $\beta_l$ greater than median. 
\State Repeat from Step 2 with $\delta_1=\frac{\delta_1}{2}$ and $\epsilon_1=\frac{3\epsilon_1}{4}$ till only one arm is remaining.

\end{algorithmic}
\end{algorithm}
 The algorithm given in Algorithm \ref{pac} finds the PAC arm for a binary reward MAB problem. If one were to follow the algorithm given in Algorithm \ref{pac}, the arm sample complexity is shown to be $O\left(\frac{2L+1}{\epsilon^2} log\left(\frac{1}{\delta}\right) \right)$ . This strategy would get us an arm that is $\epsilon$ away from the desired $\alpha$ with probability $(1-\delta)$. { While this algorithm does get us to the PAC solution, it does not utilize the obvious structure in the MAB as shown in \eqref{eq:struct}. The Algorithm \ref{pac} is formulated for MABs where the arms are independent of one another.  However we know that $\beta_i \geq \beta_{i+1} \forall i$ making the arms dependent i.e., the knowledge about a particular arm also gives us some knowledge about the other arms. In the next section we propose an algorithm which is tailored to exploit this specific dependency and reduce arm sample complexity.}
 \section{Proposed PAC solution with Reduced Sampling} \label{pacst}
We now use the structure in the MAB problem considered by us and provide an MAB solution exploiting the structure of the problem considered by us.
Since from \eqref{eq:struct} we know that, $\beta_i \geq \beta_{i+1} \:\: \forall i$ it can be seen that 
\begin{align}
\beta_i > \beta, \implies \beta_j > \beta \:\:\:\:  -L<j<i .
\end{align}
Similarly,
\begin{align}
\beta_i < \beta, \implies \beta_j < \beta \:\:\:\:  i<j<L .
\end{align}
Thus, by sampling one arm and finding out its probability of successful decoding and comparing it with the desired probability we plan to eliminate an entire set of arms. If $\beta_i < \beta$ one can eliminate all the arms greater than $i$. Similarly, if $\beta_i > \beta$ one can eliminate all the arms lesser than $i$. This follows from the natural ordering of arms in our problem. If backing off by $i$ makes the resulting TR conservative, backing off by more than $i$ will result only in a more conservative TR and vice-versa for the optimistic TR. 

We will first describe the Algorithm \ref{pacmed} for exploration and the derive the parameters for that algorithm under the condition that, the algorithm achieves the PAC($\epsilon,\delta$) solution. We shall call this the PAC Binary Search algorithm (PBS).
\begin{algorithm}[h]
 \caption{PAC Binary Search Algorithm} \label{pacmed}
\begin{algorithmic}[1]
\State The set of $2L+1$ arms indexed from $-L$ to $L$, $l_{index}=0$ $l_{high}=L,l_{low}=-L$
\State Sample $l_{index}$, $N$ times and estimate $\hat{\beta}_{l_{index}}$.
\State If $\hat{\beta}_{l_{index}}>\beta$ eliminate all arms less than  $l_{index}$, $l_{low}=l_{index}+1$.
\State If $\hat{\beta}_{l_{index}}<\beta$ eliminate all arms greater than  $l_{index}$, $l_{high}=l_{index}-1$.
\State Now re-compute $l_{index}$ to be the median of the remaining arms from $\{l_{low},l_{low}+1, \hdots l_{high}\}$ and repeat from 2 if $l_{high}>l_{low}$.
\end{algorithmic}
\end{algorithm}
In this algorithm, one needs to estimate the arm closest to our desired BLER, since we have an MAB with ordered reward probabilities, we can the binary search technique for the same. The general principle of a binary search is as follows:
In a sorted list to find a particular value, we first compare the value with the median and if the value is greater than the median we can discard all those values less than or equal to the median and implement this technique again. This technique will help us locate the desired value with just $log_2(n)$ comparisons in a list of $n$ elements when the list is sorted as against $n$comparisons in an unsorted list. In our problem, we have a set of arms whose BLER is ordered much like the sorted list and hence the same principle applies. In this situation, all we must do is to pull $k$ arms such that:
\begin{align}
k=log_2(2L+1).
\end{align}
Next we have to find out a number $N$ which is the number of times we have to pull the arms in order to ensure that the PAC criteria are satisfied. We now state and prove a theorem which we use for computing $N$. \\ 

\begin{thm}
 In order to select an arm $l$ whose $\beta_l$ is $\epsilon$ close to the optimal arm $\beta$ with probability $(1-\delta)$, the median arms from the PBS algorithm must be sampled $N$ times such that $N=max\left(\frac{ln(\frac{log_2(2L+1)}{\delta})}{\beta ln\left(\frac{\beta}{\beta-\epsilon}\right)+(1-\beta) ln\left(\frac{1-\beta}{1-\beta+\epsilon}\right)},\frac{ln(\frac{log_2(2L+1)}{\delta})}{\beta ln\left(\frac{\beta}{\beta+\epsilon}\right)+(1-\beta) ln\left(\frac{1-\beta}{1-\beta-\epsilon}\right)}\right)$. 
\end{thm}
\begin{proof}
In the PBS algorithm given as Algorithm \ref{pacmed}, we sample one arm and reject all the arms that are either greater than or lesser than that arm. If the $\epsilon$ optimal arm is in the set of rejected arms, the algorithm would never achieve $\epsilon$ optimality and we would like to limit this probability to $\delta$. Since we pull $k=log_2(2L+1)$ arms, we have to avoid eliminating the $\epsilon$ optimal arms at each step with a probability $1-\delta_1$ such that the overall probability of converging to the $\epsilon$ optimal arm is $(1-\delta)$. 

If at a given step the probability of eliminating the desired arms is $\delta_1$, the probability of eliminating the best arm after $k$ steps is $k\delta_1$ because the elimination events at different steps are independent of each other. Therefore, for ensuring the $(1-\delta)$ compliance, the probability of eliminating the desired arms at each step should be upper bounded by:
\begin{align}
\delta_1=\frac{\delta}{log_2(2L+1)} \label{eq:delc} .
\end{align}

Now let us consider the two possible events where the $\epsilon$ best arms are eliminated after sampling the arm $l$.
\begin{enumerate}[{Case} 1 :]
\item $\hat{\beta}_l>\beta$ while the true ${\beta}_l < \beta - \epsilon$. When $\hat{\beta}_l>\beta$ happens, all the arms below $l$ are discarded and if the ${\beta}_l$ is less than $\beta$ by more than $\epsilon$, it means that the $\epsilon$ optimal arm is less than $l$ but it has been eliminated.
\item $\hat{\beta}_l<\beta$ while the true ${\beta}_l > \beta + \epsilon$. Similarly, when  $\hat{\beta}_l<\beta$ happens and the ${\beta}_l$ is greater than $\beta$ by more than $\epsilon$, the optimal arm which is actually greater than $l$ gets eliminated.
\end{enumerate}

We would like to limit the probability of both the above events to $\delta_1$ and hence would like to choose a value of $N$ that ensures the same.
If $\beta_l<\beta$ as in the case 1) mentioned above we would like to ensure that:
\begin{align}
P(\hat{\beta}_l>\beta) \leq \delta_1 .
\end{align}
The value $\hat{\beta}_l$ is estimated after pulling the arm $N$ times as follows:
\begin{align}
\hat{\beta}_l=\sum_{i=1}^N \frac{\one_{ACK_i}}{N}, \label{eq:alpest}
\end{align}
where $ACK_i$ is the event that the $i$-th pull of the $l$-th arm resulted in a successful transmission. It can be seen that pulling the arm $N$ times, with the probability of success in each pull being $\beta_l$ is a binomial experiment and $\hat{\beta}_l$ is a random variable generated by the binomial experiment. The binomial random variable in the experiment is defined as :
\begin{align}
S_l=\sum_{i=1}^N \one_{ACK_i}, \label{eq:bir}.
\end{align}
The $P(\hat{\beta}_l>\beta)$ is given by:
\begin{align}
P(\hat{\beta}_l>\beta)&=P(S_l>N\beta)\\
&=P(\frac{S_l}{N}>\beta), \label{eq:bo1}
\end{align}

We first provide two tail bounds based on the Chernoff bounds and Large deviations theory \cite{cover2012elements,herbrich1999exact}. When $S_l$ is a binomial variable with parameters $N$ (number of trials) and $\beta_l$ probability of success, then :
 \ifCLASSOPTIONtwocolumn
 \begin{align}
 & P\left(\frac{S_l}{N} \geq \beta \right) \leq \\ \nonumber & exp\left(-N\left(\beta \:\:\: ln\left(\frac{\beta}{\beta_l}\right) +(1-\beta) ln\left(\frac{1-\beta}{1-\beta_l}\right)\right)\right) \nonumber \\ & \forall \:\:\: \beta > \beta_l \label{eq:rt}\\
 & P\left(\frac{S_l}{N} \leq  \beta \right) \leq \\ \nonumber & exp\left(-N\left(\beta \:\:\: ln\left(\frac{\beta}{\beta_l}\right)+ (1-\beta) ln\left(\frac{1-\beta}{1-\beta_l}\right)\right)\right) \nonumber \\ & \forall \:\:\: \beta < \beta_l \label{eq:lt}
\end{align}
 \else
\begin{align}
 P\left(\frac{S_l}{N} \geq \beta \right) \leq exp\left(-N\left(\beta \:\:\: ln\left(\frac{\beta}{\beta_l}\right)+(1-\beta) ln\left(\frac{1-\beta}{1-\beta_l}\right)\right)\right) \forall \:\:\: \beta > \beta_l \label{eq:rt}\\
 P\left(\frac{S_l}{N} \leq \beta \right) \leq exp\left(-N\left(\beta \:\:\: ln\left(\frac{\beta}{\beta_l}\right)+(1-\beta) ln\left(\frac{1-\beta}{1-\beta_l}\right)\right)\right) \forall \:\:\: \beta < \beta_l \label{eq:lt}
\end{align}
\fi
The equations \eqref{eq:rt} and \eqref{eq:lt} are respectively the right and left tail probability bounds of a binomial distribution. If a binomial experiment is conducted with parameter $p$, large deviations theory is used to compute the probability that conducting the experiment gives us a large deviation from $p$. This is precisely the problem we wish to solve. We would like to know the probability with which, an arm with success probability less than $(\beta-\epsilon)$ can masquerade as an arm with success probability greater than $\beta$ causing undesirable eliminations. This event occurs when the number of successes which is a binomial random variable (i.e., $S_l$) can deviate greatly ($>N\beta$)  from its mean value (which is $<\beta-\epsilon$).  Utilizing \eqref{eq:lt} and \eqref{eq:rt} will lead to significant reduction in arm sample complexity compared to \cite{even2002pac}. We are able to use this bound because here we know exactly the desired probability of reward.
  Under our assumptions, the ACK process is a binomial process. 
We can apply the right-tail bound on \eqref{eq:bo1}, and since we have an upper bound all we need to ensure is to bound the upper bound by $\delta_1$.
 \ifCLASSOPTIONtwocolumn
 \begin{align}
 &P\left(\frac{S_l}{N} \geq \beta \right) \nonumber \\ &\leq exp\left(-N\left(\beta \:\:\: ln\left(\frac{\beta}{\beta_l}\right)+(1-\beta) ln\left(\frac{1-\beta}{1-\beta_l}\right)\right)\right) \nonumber \\ &\leq \delta_1 \label{eq:ub}
\end{align}
 \else
\begin{align}
 P\left(\frac{S_l}{N} \geq \beta \right) \leq exp\left(-N\left(\beta \:\:\: ln\left(\frac{\beta}{\beta_l}\right)+(1-\beta) ln\left(\frac{1-\beta}{1-\beta_l}\right)\right)\right)\leq \delta_1 \label{eq:ub}
\end{align}
\fi
Our value of $N$ is to be such that we do not wrongly eliminate the set of arms whose actual probability of success is higher than $\beta-\epsilon$. Such an elimination will occur when for the arm $l$, $\beta_l<\beta-\epsilon$ is wrongly estimated as greater than $\beta$.
Let $\F$ be the event that $\beta_l \in (0,\beta-\epsilon)$ and $p_{\F}$ be the probability that, $P(\frac{S_l}{N}\geq\beta|\F)$, we would like to bound the $p_{\F}$ by $\delta_1$ for obtaining the PAC solution and we start with an upper bound of $p_{\F}$ itself :
\begin{align}
 sup\{ p_{\F}\} = P(\frac{S_l}{N}\geq\beta|\beta_l=\beta-\epsilon).
\end{align}
This is an upper bound when $S_l$ is binomially distributed and as $\beta_l$ increases the probability of getting atleast $N\beta$ successes increases and from the event space $\F$ we can see that the value of $\beta_l$ is upper bounded by $(\beta-\epsilon)$.
Therefore, if
\begin{align}
P(\frac{S_l}{N}\geq\beta|\beta_l=\beta-\epsilon) < \delta_1 \label{eq:fini1},
\end{align}
then
\begin{align}
p_{\F} \leq \delta_1 .
\end{align}
Thus if we solve \eqref{eq:fini1} for $N$ we will achieve the PAC solution. We also know from \eqref{eq:rt} that,
 \ifCLASSOPTIONtwocolumn
 \begin{align}
 &P(\frac{S_l}{N}\geq\beta|\beta_l=\beta-\epsilon) \leq \nonumber \\ & exp\left(-N\left(\beta \: ln\left(\frac{\beta}{\beta-\epsilon}\right)+(1-\beta) ln\left(\frac{1-\beta}{1-\beta+\epsilon}\right)\right)\right), 
\end{align}
\else
\begin{align}
 P(\frac{S_l}{N}\geq\beta|\beta_l=\beta-\epsilon) \leq exp\left(-N\left(\beta \: ln\left(\frac{\beta}{\beta-\epsilon}\right)+(1-\beta) ln\left(\frac{1-\beta}{1-\beta+\epsilon}\right)\right)\right), 
\end{align}
\fi
and hence,
 \ifCLASSOPTIONtwocolumn
\begin{align}
&e^{\left(-N\left(\beta \: ln\left(\frac{\beta}{\beta-\epsilon}\right)+(1-\beta) ln\left(\frac{1-\beta}{1-\beta+\epsilon}\right)\right)\right)
} \leq \delta_1 \label{eq:fins}\\
 &\implies P(\frac{S_l}{N}\geq\beta|\beta_l=\beta-\epsilon) \leq \delta_1, \\
 &\implies p_{\F} \leq \delta_1.
\end{align}
\else
\begin{align}
&exp\left(-N\left(\beta \: ln\left(\frac{\beta}{\beta-\epsilon}\right)+(1-\beta) ln\left(\frac{1-\beta}{1-\beta+\epsilon}\right)\right)\right) \leq \delta_1 \label{eq:fins}\\
 &\implies P(\frac{S_l}{N}\geq\beta|\beta_l=\beta-\epsilon) \leq \delta_1, \\
 &\implies p_{\F} \leq \delta_1.
\end{align}
\fi
Solving for $N$ in \eqref{eq:fins}, we obtain
\begin{align}
 N &\geq \frac{ln(\frac{1}{\delta_1} )}{\beta ln\left(\frac{\beta}{\beta-\epsilon}\right)+(1-\beta) ln\left(\frac{1-\beta}{1-\beta+\epsilon}\right)} \\
 N &\geq \frac{ln(\frac{log_2(2L+1)}{\delta})}{\beta ln\left(\frac{\beta}{\beta-\epsilon}\right)+(1-\beta) ln\left(\frac{1-\beta}{1-\beta+\epsilon}\right)} \label{eq:frt} .
\end{align}

Similarly if we assign $\e$ to be the event that $\beta_l>\beta+\epsilon$, we would want to limit $p_{\e}=P(\frac{S_l}{N}<\beta|\E)$ to less than $\delta_1$. The value $\p_{\e}$ is upper bounded by
\begin{align}
 sup\{ p_{\e}\} = P(\frac{S_l}{N}\leq\beta|\beta_l=\beta+\epsilon). \label{eq:siml}
\end{align}
Now we are dealing with left tail probabilities and the probability of getting lesser than $N\beta$ successes increase with decreasing $\beta_l$ and since $\beta+\epsilon$ is the lowest $\beta_l \in \e$ \eqref{eq:siml} is true.
Now following similar steps using \eqref{eq:lt} we obtain 
\begin{align}
 N &\geq \frac{ln(\frac{1}{\delta_1}  )}{\beta ln\left(\frac{\beta}{\beta+\epsilon}\right)+(1-\beta) ln\left(\frac{1-\beta}{1-\beta-\epsilon}\right)} \\
 N &\geq \frac{ln(\frac{log_2(2L+1)}{\delta})}{\beta ln\left(\frac{\beta}{\beta+\epsilon}\right)+(1-\beta) ln\left(\frac{1-\beta}{1-\beta-\epsilon}\right)} \label{eq:flt} .
\end{align}
In \cite{even2002pac}, as a consequence of not having a desired target probability, the above bound cannot be used and hence the value of $N \propto \frac{1}{\epsilon^2}$. Here, since we applied a bound which implied knowing the target success probability we get a much lower sample complexity.
Finally from \eqref{eq:frt} and \eqref{eq:flt} we obtain:
\ifCLASSOPTIONtwocolumn
\begin{align}
 N=&max\left(\frac{ln(\frac{log_2(2L+1)}{\delta})}{\beta ln\left(\frac{\beta}{\beta-\epsilon}\right)+(1-\beta) ln\left(\frac{1-\beta}{1-\beta+\epsilon}\right)},\right.\nonumber \\ & \left. \frac{ln(\frac{log_2(2L+1)}{\delta})}{\beta ln\left(\frac{\beta}{\beta+\epsilon}\right)+(1-\beta) ln\left(\frac{1-\beta}{1-\beta-\epsilon}\right)}\right).\label{eq:limn}
\end{align}
\else
\begin{align}
 N=max\left(\frac{ln(\frac{log_2(2L+1)}{\delta})}{\beta ln\left(\frac{\beta}{\beta-\epsilon}\right)+(1-\beta) ln\left(\frac{1-\beta}{1-\beta+\epsilon}\right)},\frac{ln(\frac{log_2(2L+1)}{\delta})}{\beta ln\left(\frac{\beta}{\beta+\epsilon}\right)+(1-\beta) ln\left(\frac{1-\beta}{1-\beta-\epsilon}\right)}\right).\label{eq:limn}
\end{align}
\fi
\end{proof}
We would like to point out some inferences on the sample complexity reduction achieved. For $\beta=0.9$, $\epsilon=0.05$, $\delta=0.05$, and $L=3$ the standard algorithm discussed has a sample complexity of 2000 samples per arm and 14000 samples overall whereas the proposed PBS algorithm has a sample complexity of only 375 per arm and 1125 overall which is a reduction in sample complexity by a little more than a factor of 12. This reduction is possible only because of the specific structure of the problem at hand which we exploit owing to better theoretical bounds.
\subsection{Further Reductions in the Sample Complexity} \label{pacon}
An algorithm to reduce sample complexity was proposed in the previous section exploiting the dependant arm structure and the knowledge of the optimal success probability. However the reader may have noticed  the estimated value $\hat{\beta}_l$ was not exploited to reduce $N$. For instance if the desired value is $\beta=0.9$, and after some $50$ samples the estimated $\hat{\beta}_l$ is $0.5$, the probability of the true $\beta_l$ being close to $0.9$ is very low. We can use confidence intervals bound to exploit this. If after $N_1$ samples a $\hb<\beta$ and neither $\beta-\epsilon$ nor $\beta+\epsilon$ fall in the $1-\delta_1$ confidence interval of $\hb$ then one can say that $\beta_l$ is not within $(\beta-\epsilon,\beta+\epsilon)$ with probability $1-\delta_1$. We propose to use the Wald's confidence interval given in \cite{agresti1998approximate} to obtain:

\begin{align}
 \hb+z\sqrt{\frac{1}{N1}\hb(1-\hb)}<\beta-\epsilon \label{eq:ucb}\\
 \hb-z\sqrt{\frac{1}{N1}\hb(1-\hb)}>\beta+\epsilon \label{eq:lcb},
\end{align}
where $z=Q^{-1}(1-\delta_1)$ and $Q(.)$ is the Q-function corresponding to the tail probability of standard normal distribution. These confidence bounds are accurate for large values of $N_1$ and we heuristically decide to check these confidence intervals from $N\geq20$.
The L.H.S of \eqref{eq:ucb} is the upper confidence bound (UCB) \footnote{There are MAB indexing algorithms that use UCB of the reward estimate to compute the index. However, here we are using the UCB and LCB to reject arms for the $(\epsilon,\delta)$ approach. } of $\beta_l$ i.e., the true value of $\beta_l$ will exceed the UCB by a probability of only $\delta_1$. Similarly the L.H.S of \eqref{eq:lcb} is the lower confidence bound (LCB) i.e., the true value of $\beta_l$ will go lower than the LCB with a probability not more than $\delta_1$. If for a given $l$ the \eqref{eq:ucb} is satisfied then one can reject all the arms $i>l$ and vice-versa for the \eqref{eq:lcb}. Now, if $L=3$, $\beta=0.9$, $\hb=0.5$, $\epsilon=0.05$ and $\delta=0.05$ after $N_1=20$ we can stop sampling that arm and reject all the arms with back-off lesser than that arm. This can greatly reduce the sample complexity and the PBS algorithm can be modified by checking conditions in \eqref{eq:ucb} and \eqref{eq:lcb} after the number of samples crosses 20 and if either of the conditions are satisfied, we reject the arms as mentioned previously else, we will continue as in Algorithm \ref{pacmed}. This approach would help us to recognize a cluster of ACK/NACK inputs and adapt our strategy accordingly if we find an overwhelming number of continuous ACKs or NACKs. 
\subsection{Switching Phase during Exploitation}{\label{swi}}
 Now it may be the case that there is no single offset value that can achieve the exact target success probability $\beta$. However if there is an arm $i$ that achieves success probability between $\beta-\epsilon$ and $\beta$, we can sample another arm $i-1$ so as to achieve $\beta$. Similarly if there is an arm $i$ that achieves success probability between $\beta$ and $\beta+\epsilon$ we can sample another arm $i+1$ to get closer to $\beta$. Therefore, we switch between two arms such that the success probability is as close to the target probability as possible.
\begin{algorithm}[h]
 \caption{Switching Algorithm} \label{switchin}
\begin{algorithmic}[1]
\State The arm $i$ after exploration with estimated probability $\hat{\beta}_i$, $\hat{\beta}=\hat{\beta}_i$.
\State If $\hat{\beta}<\beta$ Set of arms $S=\{i-1,i\}$.
\State If $\hat{\beta}>\beta$ Set of arms $S=\{i,i+1\}$.
\State If $\hat{\beta}<\beta$ sample arm $S(1)$ else sample arm $S(2)$.
\State Estimate $\hat{\beta}$ based on current ACK/NACK using \eqref{eq:alpestin}.
\end{algorithmic}
\end{algorithm}

We use Theorem 1, \ref{pacon} and \ref{swi} to come up with the final Algorithm \ref{fin}.

\subsection{Final Algorithm}
\ifCLASSOPTIONtwocolumn
\begin{algorithm}[h]
 \caption{Final Algorithm} \label{fin}
\begin{algorithmic}[1]
\State The set of $2L+1$ arms indexed from $-L$ to $L$, $l_{low}=-L$ ,$l_{high}=L$ 
\State If$l_{high}>l_{low}$:$l_{index}=median(l_{low}:l_{high})$ and Sample $l_{index}$. Else go to Step 7.
\State If after $N$ samples where $N>20$ and $\hb+z\sqrt{\frac{1}{N}\hb(1-\hb)}<\beta-\epsilon$, $l_{high}=l_{index}-1$ and goto Step 2.
\State If after $N$ samples where $N<20$ and and $\hb-z\sqrt{\frac{1}{N}\hb(1-\hb)}>\beta+\epsilon$, $l_{low}=l_{index}+1$ and goto Step 2.

\State If after $N=$ 
 \Statex $max\left(\frac{ln(\frac{log_2(2L+1)}{\delta})}{\beta ln\left(\frac{\beta}{\beta-\epsilon}\right)+(1-\beta) ln\left(\frac{1-\beta}{1-\beta+\epsilon}\right)}\right.,$ \Statex $\left.\frac{ln(\frac{log_2(2L+1)}{\delta})}{\beta ln\left(\frac{\beta}{\beta+\epsilon}\right)+(1-\beta) ln\left(\frac{1-\beta}{1-\beta-\epsilon}\right)}\right)$, $\hb<\beta-\epsilon$ perform $l_{high}=l_{index}-1$ and goto Step 2.
\State If after $N$ samples where $N=max\left(\frac{ln(\frac{log_2(2L+1)}{\delta})}{\beta ln\left(\frac{\beta}{\beta-\epsilon}\right)+(1-\beta) ln\left(\frac{1-\beta}{1-\beta+\epsilon}\right)}\right.$,\Statex $\left.\frac{ln(\frac{log_2(2L+1)}{\delta})}{\beta ln\left(\frac{\beta}{\beta+\epsilon}\right)+(1-\beta) ln\left(\frac{1-\beta}{1-\beta-\epsilon}\right)}\right)$, $\hb>\beta+\epsilon$ $l_{low}=l_{index}+1$ and goto Step 2.
\State If after $N=max\left(\frac{ln(\frac{log_2(2L+1)}{\delta})}{\beta ln\left(\frac{\beta}{\beta-\epsilon}\right)+(1-\beta) ln\left(\frac{1-\beta}{1-\beta+\epsilon}\right)},\right.$\Statex$\left.\frac{ln(\frac{log_2(2L+1)}{\delta})}{\beta ln\left(\frac{\beta}{\beta+\epsilon}\right)+(1-\beta) ln\left(\frac{1-\beta}{1-\beta-\epsilon}\right)}\right)$ and $\beta-\epsilon<\hb<\beta$ initialize ${ l_{switching}}= \{l_{index}$, $l_{index}-1\}$ and goto Step 9.

\State If after $N=max\left(\frac{ln(\frac{log_2(2L+1)}{\delta})}{\beta ln\left(\frac{\beta}{\beta-\epsilon}\right)+(1-\beta) ln\left(\frac{1-\beta}{1-\beta+\epsilon}\right)},\right.$\Statex$\left.\frac{ln(\frac{log_2(2L+1)}{\delta})}{\beta ln\left(\frac{\beta}{\beta+\epsilon}\right)+(1-\beta) ln\left(\frac{1-\beta}{1-\beta-\epsilon}\right)}\right)$ and $\beta<\hb<\beta+\epsilon$ initialize $l_{switching}= \{l_{index}$, $l_{index}+1\}$ and goto Step 9.

\State If $\hb>\beta$ use $l_{index}=max({ l_{switching}})$ else use $l_{index}=min({ l_{switching}})$ and repeat step 9 till end. 
\end{algorithmic}
\end{algorithm}
\else
\begin{algorithm}[h]
 \caption{Final Algorithm} \label{fin}
\begin{algorithmic}[1]
\State The set of $2L+1$ arms indexed from $-L$ to $L$, $l_{low}=-L$ ,$l_{high}=L$ 
\State If$l_{high}>l_{low}$:$l_{index}=median(l_{low}:l_{high})$ and Sample $l_{index}$. Else go to Step 7.
\State If after $N$ samples where $N>20$ and $\hb+z\sqrt{\frac{1}{N}\hb(1-\hb)}<\beta-\epsilon$, $l_{high}=l_{index}-1$ and goto Step 2.
\State If after $N$ samples where $N<20$ and and $\hb-z\sqrt{\frac{1}{N}\hb(1-\hb)}>\beta+\epsilon$, $l_{low}=l_{index}+1$ and goto Step 2.

\State If after $N=max\left(\frac{ln(\frac{log_2(2L+1)}{\delta})}{\beta ln\left(\frac{\beta}{\beta-\epsilon}\right)+(1-\beta) ln\left(\frac{1-\beta}{1-\beta+\epsilon}\right)},\frac{ln(\frac{log_2(2L+1)}{\delta})}{\beta ln\left(\frac{\beta}{\beta+\epsilon}\right)+(1-\beta) ln\left(\frac{1-\beta}{1-\beta-\epsilon}\right)}\right)$, $\hb<\beta-\epsilon$ perform $l_{high}=l_{index}-1$ and goto Step 2.
\State If after $N$ samples where $N=max\left(\frac{ln(\frac{log_2(2L+1)}{\delta})}{\beta ln\left(\frac{\beta}{\beta-\epsilon}\right)+(1-\beta) ln\left(\frac{1-\beta}{1-\beta+\epsilon}\right)},\frac{ln(\frac{log_2(2L+1)}{\delta})}{\beta ln\left(\frac{\beta}{\beta+\epsilon}\right)+(1-\beta) ln\left(\frac{1-\beta}{1-\beta-\epsilon}\right)}\right)$, $\hb>\beta+\epsilon$ $l_{low}=l_{index}+1$ and goto Step 2.
\State If after $N=max\left(\frac{ln(\frac{log_2(2L+1)}{\delta})}{\beta ln\left(\frac{\beta}{\beta-\epsilon}\right)+(1-\beta) ln\left(\frac{1-\beta}{1-\beta+\epsilon}\right)},\frac{ln(\frac{log_2(2L+1)}{\delta})}{\beta ln\left(\frac{\beta}{\beta+\epsilon}\right)+(1-\beta) ln\left(\frac{1-\beta}{1-\beta-\epsilon}\right)}\right)$ and $\beta-\epsilon<\hb<\beta$ initialize ${ l_{switching}}= \{l_{index}$, $l_{index}-1\}$ and goto Step 9.

\State If after $N=max\left(\frac{ln(\frac{log_2(2L+1)}{\delta})}{\beta ln\left(\frac{\beta}{\beta-\epsilon}\right)+(1-\beta) ln\left(\frac{1-\beta}{1-\beta+\epsilon}\right)},\frac{ln(\frac{log_2(2L+1)}{\delta})}{\beta ln\left(\frac{\beta}{\beta+\epsilon}\right)+(1-\beta) ln\left(\frac{1-\beta}{1-\beta-\epsilon}\right)}\right)$ and $\beta<\hb<\beta+\epsilon$ initialize $l_{switching}= \{l_{index}$, $l_{index}+1\}$ and goto Step 9.

\State If $\hb>\beta$ use $l_{index}=max({ l_{switching}})$ else use $l_{index}=min({ l_{switching}})$ and repeat step 9 till end. 
\end{algorithmic}
\end{algorithm}
\fi
 In Algorithm \ref{fin} we first pick $l_{index}$ as the arm to be sampled at Step 2. Then after sampling it a sufficient number of times say 20, we estimate the upper and lower confidence bounds of probability estimate $\hat{\beta_l}$ with probability $(1-\delta_1)$. If the confidence interval lies completely outside the interval $(\beta-\epsilon,\beta+\epsilon)$ then we decide reject all arms $<l$ or $>l$ if the $LCB>(\beta+\epsilon)$ or $UCB<(\beta-\epsilon)$ as the case may be. These set of operations are performed in steps 3 and 4. Thus when the estimated $\hat{\beta}_l$ for a given arm $l$ is much less than $\beta-\epsilon$ or much greater than $\beta+\epsilon$ after 20 samples based on the confidence bounds, we perform the elimination of all the arms which are even further from the desired interval $(\beta-\epsilon,\beta+\epsilon)$ conditioned on $\hat{\beta}_l$, without waiting for the entire sample complexity $N$ as in \eqref{eq:limn}. Otherwise, we obtain $N$ samples and the algorithm performs Steps 5 and 6 similar to Algorithm \ref{pacmed}. In Section \ref{pacst}, we saw the arm sample complexity to be 1125 for a typical case. Using the confidence bounds we can reduce the arm sample complexity even further.  Finally after we reduce the set of arms to a single final arm, we enter the switching phase in Steps 7,8 and 9 which has been described completely in Section \ref{swi}.  
 \section{Other Approaches for solving the MAB}
 \subsection{Thompson Sampling Approach}
 The Thompson sampling based solution from \cite{chapelle2011empirical} is given in Algorithm \ref{TS}.
 \begin{algorithm}[h]
 \caption{Thompson Sampling Algorithm} \label{TS}
\begin{algorithmic}[1]
\State For the set of $2L+1$ arms from $l \in {-L,L}$, $S_l=F_l=0$
\State At each time instant $t$, for each $l$ generate $\theta_l(t)$ from the pdf $f(x,S_l,F_l)=\frac{\Gamma(S_l+F_l)}{\Gamma(S_l)\Gamma(F_l)}x^{S_l-1}(1-x)^{F_l-1}$
\State Use the offset $l$ whose $\theta_l(t)$ is the highest.
\State If ACK increment $S_l$ else increment $F_l$.
\State Repeat from Step 2 with new $S_l$s and $F_l$s.

\end{algorithmic}
\end{algorithm}
In the algorithm given, initially, all the $\theta_l$s are from the same uniform distribution then as the successes increase for a given $l$ the $f(x,S_l,F_l)$ will pick higher and higher values. However for our problem this only means than the most conservative arm will be picked.
\subsection{UCB Approach}
 The UCB based algorithm from \cite{auer2002finite} is given in Algorithm \ref{ucbg}.
\begin{algorithm}[h]
 \caption{UCB Algorithm} \label{ucbg}
\begin{algorithmic}[1]
\State For the set of $2L+1$ arms play each arm once
\State Let $x_l=\frac{S_l}{N_l}$ be the success probability of each arm where $S_l$ is the number of successes and $N_l$ is the number of failures.
\State Use the the offset $l$ for which $x_l+\sqrt{\frac{2 log(n)}{N_l}}$ is maximum where $n$ is the total number of trials across all offset
\State If ACK increment $S_l,N_l$ else increment $N_l$.

\end{algorithmic}
\end{algorithm}
This algorithm will also converge to the offset that maximizes the success probability. Both these algorithms as we can see converge to a very conservative offset. There could be techniques where instead of picking the ``best'' arm as estimated by these algorithms we could pick up a second or third best arm in terms of success probabilities, however the mapping between the arm ranking and proximity to the target BLER is neither obvious nor trivial.  While the vanilla UCB and Thompson based approaches fail to achieve our targeted objective, there is scope for future work in other variants based on UCB/Thompson sampling which could get us to a desired success probability rather than simply maximizing reward. However, in the current contribution the above reasons provide justification as to why our PAC based technique is ideally suited to this problem. 

\section{Simulations and Results}
An LTE cell with 7 cell - 3 sectors is considered, for simulation with ``wrap around'' implemented, in order to avoid edge discontinuities. The UEs are distributed uniformly in each sector with 15 UEs per sector.  LTE systems use OFDMA in the physical layer where sub-carriers are grouped into sub-bands \cite{lte36211}, and users are allocated a set of sub-bands for data transmission. Each eNodeB transmits over the same set of resources, as, it is a reuse-one system. The OFDMA for the 10MHz LTE system has 1024 sub-carriers where only the 600 in the middle are used \cite{lte36213}. These 600 sub-carriers  are grouped into 50 groups of 12 sub-carriers (SCs) each and this is done over 14 OFDM symbols. So this group of 12 SCs over 14 symbols is called one Physical 
Resource Block(PRB) and the 14 OFDM symbols together constitute a sub-frame \cite{lte36211}. There are 50 PRBs in a sub-frame and a continuous block of 3PRBs are grouped to form a sub-band. There are 17 sub-bands \footnote{16 of these sub-bands consists of 3 frequency domain contiguous PRBs, while the 17th sub-band consists of only 2 contiguous PRBs.} in LTE for the 10MHz system \cite{lte36211}, and, scheduling and transmission is done at the sub-band level. The set of sub-bands allocated to a user, is called a transport block and every user will be allocated one rate for the whole transport block. For more details on the LTE system, please refer to \cite{pull2015}.

The path loss exponent and shadow fading parameters are as specified in \cite{scm} for an Urban Macro model. The channel model used in the simulator is the generic channel model for wideband channels as specified in \cite{scm} which is a realistic Spatial Channel Model  for wideband cellular systems. The channels between each UE and the 9 strongest eNodeBs are modelled using different parameters such as the angle of arrival and departure of the multipath rays, distance-dependent power delay profile, and multipath profiles \cite{scm}. Thus we are modelling 8 strongest interferers apart from the desired eNodeB. We find that for every UE the interferers which are weaker than the 8-th strongest interferer are below the level of thermal noise and hence we do not model them explicitly to save simulation complexity. It can be seen that, different users experience different delay spreads, and even the same user is subjected to different delay spreads from different eNodeBs. Hence, the multipath power delay profile of the channel between the UE and the serving eNodeB can differ from that between the served UE and the interfering eNodeBs. 

A round robin scheduler is used and the UEs feedback wideband CQI as in \cite{lte36213} Periodic Mode 1-0. The detailed simulation set-up is provided in Table \ref{table:bsm}. Typically, LTE system requires a BLER of $10^{-1}$ or lesser. Hence we study the system performance with target BLERs in that range: a) $0.075$, b) $0.1$, c) $0.125$, d) $0.15$.
\ifCLASSOPTIONtwocolumn
\begin{table}
\caption{Baseline Simulation Parameters}
  \centering
    \begin{tabular}{|l|r|}
        \hline
        Duplex method & FDD \\ \hline
        Test Environment & Base coverage urban \\ \hline
        Deployment scenario & Urban macro-cell scenario \\ \hline
        Base station antenna height & 25 m, above rooftop \\ \hline
        Minimum distance between&\\ UT and serving cell & $ >= 25m $ \\ \hline
       Outdoor to in-car penetration loss & 9 dB(LN,$ \sigma = 5dB)$ \\ \hline
        Layout & seven cell Hexagonal \\&grid with wrap around. \\ \hline
        carrier frequency & 2 GHz \\ \hline
        Inter-site distance & 500 m \\ \hline
        UT speeds of interest & 30 km/h \\ \hline
        Total eNodeB transmit power & 46 dBm for 10 MHz \\ \hline
        Thermal noise level & -174 dBm/Hz \\ \hline
        User mobility model & Fixed and identical \\& speed $ |v| $ of all UTs,\\
& randomly and uniformly\\& distributed direction \\ \hline
        Inter-site interference modelling & 8 interferers Explicitly modelled \\ \hline
       eNodeB noise figure & 5 dB \\ \hline
        UT noise figure & 7 dB \\ \hline
        eNodeB antenna gain & 17 dBi \\ \hline
        UT antenna gain & 0 dBi \\ \hline
        Antenna configuration & Vertically polarized antenna \\& 0.5 wavelength separation \\
& at UE  10 wavelength \\&separation at base-station \\    \hline
        Channel Model & Urban Macro model (UMa) in \cite{scm} \\ \hline
           Network synchronization & Synchronized   \\ \hline
        Downlink transmission scheme & 1x2 Single Input Multiple Output       \\ \hline
        Downlink Scheduler & Round Robin with full \\& bandwidth allocation  \\ \hline
        Downlink Adaptation &  Wideband CQI  \\
& for all users,at 5 ms CQI \\& feedback periodicity,\\& CQI delay :Ideal, \\
\hline
     
        Simulation bandwidth & 10 + 10 MHz (FDD)   \\ \hline
        \hline
        
    \end{tabular}
    \label{table:bsm}
\end{table}
\else
\begin{table}
\caption{Baseline Simulation Parameters}
  \centering
    \begin{tabular}{|l|r|}
        \hline
        Duplex method & FDD \\ \hline
        Test Environment & Base coverage urban \\ \hline
        Deployment scenario & Urban macro-cell scenario \\ \hline
        Base station antenna height & 25 m, above rooftop \\ \hline
        Minimum distance between UT and serving cell & $ >= 25m $ \\ \hline
       Outdoor to in-car penetration loss & 9 dB(LN,$ \sigma = 5dB)$ \\ \hline
        Layout & seven cell Hexagonal grid with wrap around. \\ \hline
        carrier frequency & 2 GHz \\ \hline
        Inter-site distance & 500 m \\ \hline
        UT speeds of interest & 30 km/h \\ \hline
        Total eNodeB transmit power & 46 dBm for 10 MHz \\ \hline
        Thermal noise level & -174 dBm/Hz \\ \hline
        User mobility model & Fixed and identical speed $ |v| $ of all UTs,\\
& randomly and uniformly distributed direction \\ \hline
        Inter-site interference modelling & 8 interferers Explicitly modelled \\ \hline
       eNodeB noise figure & 5 dB \\ \hline
        UT noise figure & 7 dB \\ \hline
        eNodeB antenna gain & 17 dBi \\ \hline
        UT antenna gain & 0 dBi \\ \hline
        Antenna configuration & Vertically polarized antenna 0.5 wavelength separation \\
& at UE  10 wavelength separation at base-station \\    \hline
        Channel Model & Urban Macro model (UMa) in \cite{scm} \\ \hline
           Network synchronization & Synchronized   \\ \hline
        Downlink transmission scheme & 1x2 Single Input Multiple Output       \\ \hline
        Downlink Scheduler & Round Robin with full bandwidth allocation  \\ \hline
        Downlink Adaptation &  Wideband CQI  \\
& for all users,at 5 ms CQI feedback periodicity, CQI delay :Ideal, \\
\hline
     
        Simulation bandwidth & 10 + 10 MHz (FDD)   \\ \hline
        \hline
        
    \end{tabular}
    \label{table:bsm}
\end{table}
\fi
The final Algorithm \ref{fin} we apply is entirely based on the principles of both Sections \ref{pacst} and \ref{pacon}. The Steps 5 and 6 initiate a switching operation. This implies that when we find an arm that is $\epsilon$ optimal we retain that arm and its neighbour. If the $\epsilon$ optimal arm $m$ has success probability less than $\beta$ we retain the arm $(m-1)$ and play that arm till the success probability touches $\beta$ and vice-versa.

Given the simulation parameters we compare 4 schemes over a simulation duration of 5 seconds:
\begin{itemize}
\item The proposed OLLA algorithm based on MAB with Target BLER as $10\%$ i.e., $\alpha=0.1,\beta=0.9$.
\item The proposed OLLA algorithm based on MAB with Target BLER as $7.5\%$ i.e., $\alpha=0.075,\beta=0.925$.
\item An OLLA technique where when 5 consecutive NACKs are seen the value $l$ in \eqref{eq:off} is decremented by $1$ and it is incremented by $1$ when 50 consecutive ACKs are seen. This is based on \cite{ariyanthe} and is reffered to as the ``Clustering'' scheme because we look at error clusters.
\item A simulation with No OLLA.
\end{itemize}

Looking at \figref{fig:bl} where we compare the BLER attained by the various schemes. It can be seen that the Clustering technique has a substantially lower BLER than MAB and no OLLA schemes. The different MAB schemes achieve their respective target BLERs with high accuracy. Around $80\%$ the users also achieve a BLER less than or equal to the desired target BLER. Thus it is seen that the MAB schemes provide us a high level of control over the BLER levels.

 \ifCLASSOPTIONtwocolumn
\begin{figure}
\begin{center}
 \includegraphics[width=0.8\columnwidth]{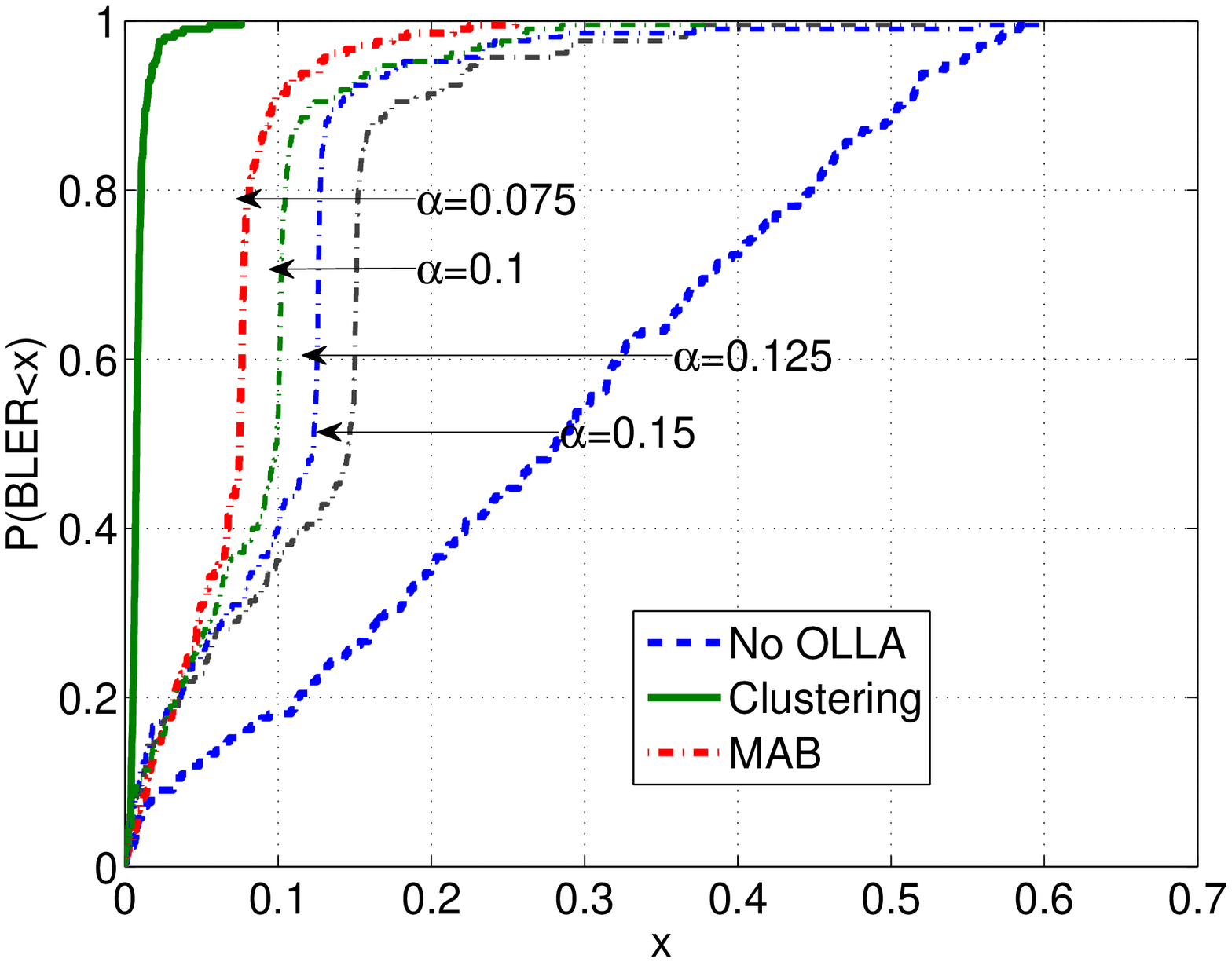}
\caption{BLER CDF}
\label{fig:bl}
\end{center}
\end{figure}
\else
 \begin{figure}
\begin{center}
 \includegraphics[width=0.6\columnwidth]{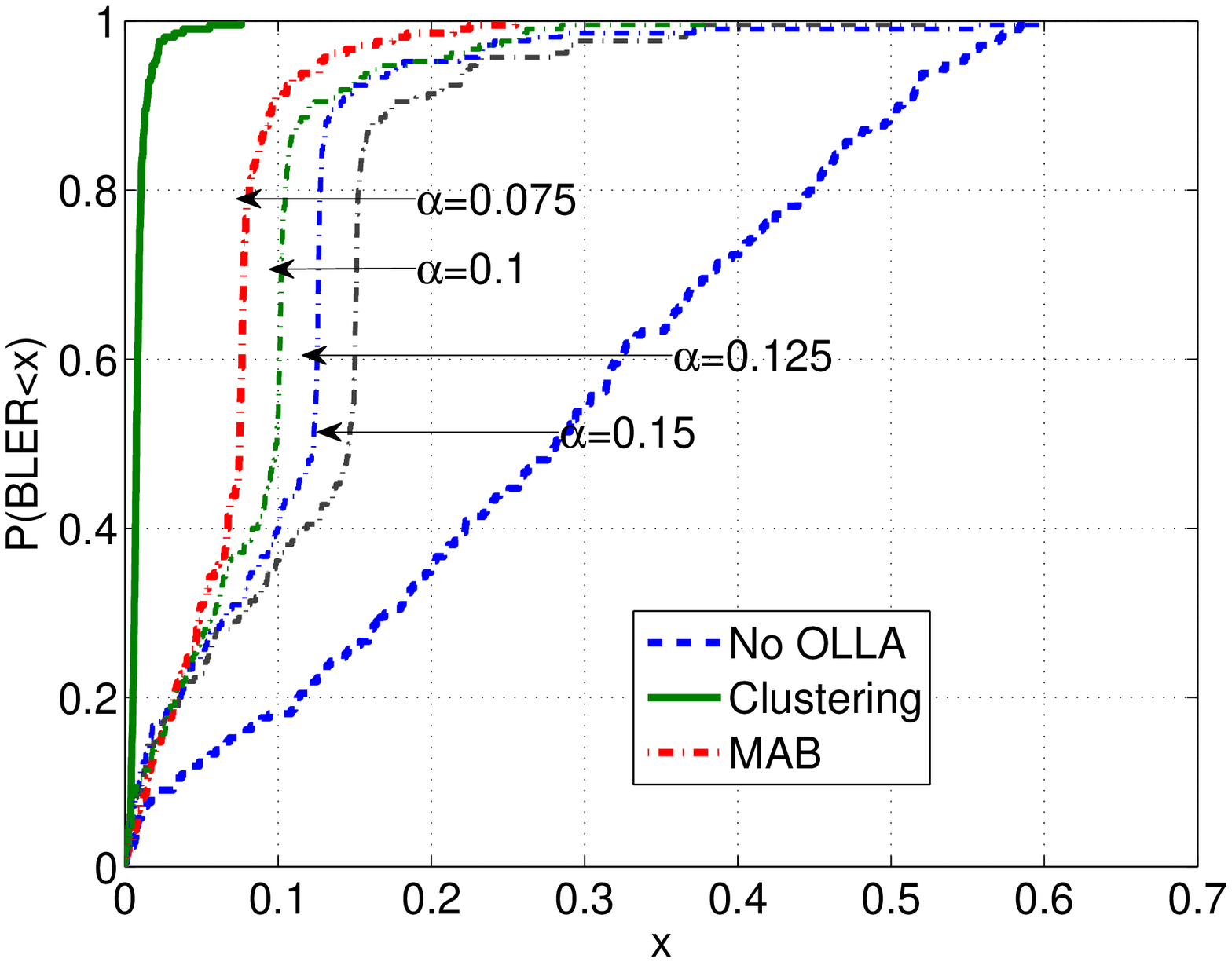}
\caption{BLER CDF}
\label{fig:bl}
\end{center}
\end{figure}
\fi

Now let us look at \figref{fig:th} where we compare the throughput obtained by the clustering scheme and the best MAB based scheme. It can be seen that 30\% of MAB users have a throughput higher than 2Mbps while only 18\% of the users under clustering show a similar performance. In fact, the throughput performance of the clustering scheme is poorer than performing no OLLA at all. Therein lies the fault with an arbitrary OLLA scheme. Our proposed scheme on the other hand is able to control BLER while simultaneously achieving a higher throughput than a no OLLA scheme. This is because the clustering tends to achieve arbitrarily low BLER at the cost of reducing the transmission rates while MAB does not try to achieve lower than necessary BLER levels. To understand this we would like to present the average offset value when clustering and MAB are used. The average offset value for clustering was less than $-4$, while it was just below $-2$ for MAB. This is because, the clustering technique reacted very fast to incoming NACKs whereas in the MAB algorithm considered we explore each offset for a while before coming to a conclusion.
 \ifCLASSOPTIONtwocolumn
\begin{figure}
\begin{center}
 \includegraphics[width=0.8\columnwidth]{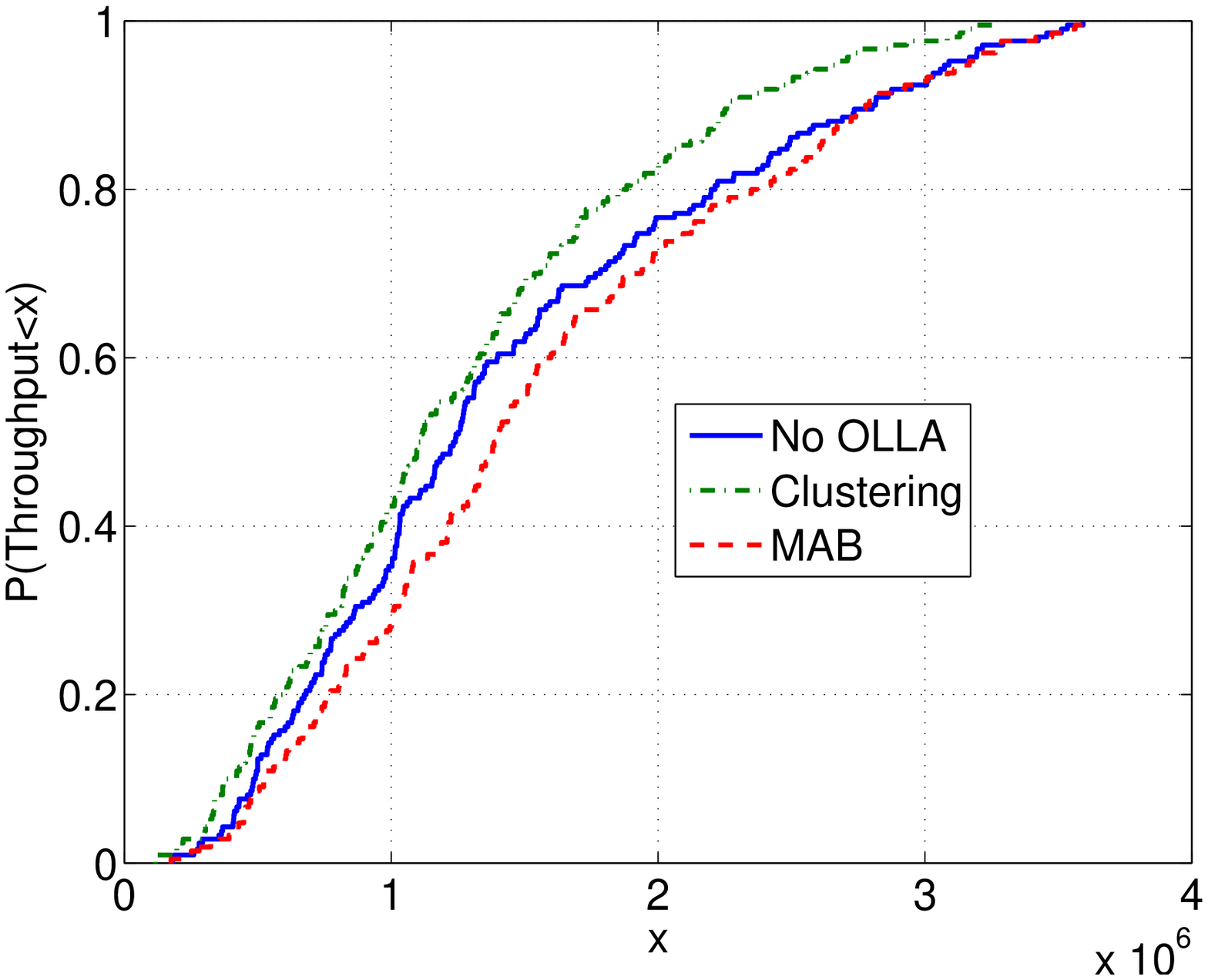}
\caption{Throughput CDF}
\label{fig:th}
\end{center}
\end{figure}
\else
 \begin{figure}
\begin{center}
 \includegraphics[width=0.6\columnwidth]{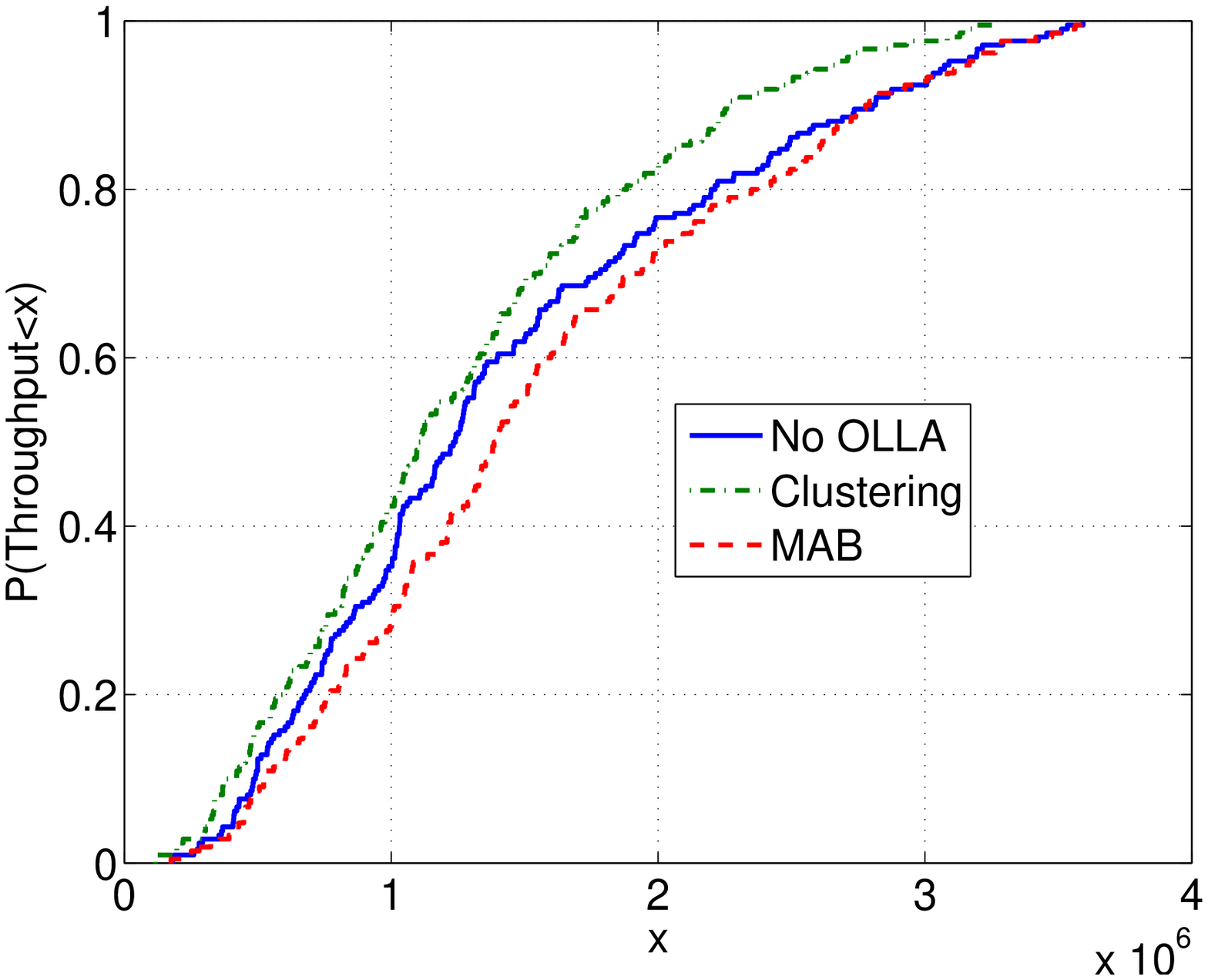}
\caption{Throughput CDF}
\label{fig:th}
\end{center}
\end{figure}
\fi

  
\begin{table}
\caption{A comparison table of the schemes with various parameters averaged}
  \centering
    \begin{tabular}{|c||c|c|c|c|}
        \hline
      Parameter  &MAB&MAB&Clustering& No OLLA  \\for Comparison & 7.5\% &10\% &  & \\ \hline \hline
      
      Average &&&&\\ Throughput &1.55 Mbps&1.54 Mbps &1.25Mbps&1.43Mbps \\ \hline
      Average BLER &6.66\%&8.51\%&1\%&27.7\% \\ \hline
      Average  &2.39&2.41&2.15&2.24 \\ 
      Symbol Rate & & & & \\ \hline
    \end{tabular}
\label{table:com}
\end{table}
The average throughput, BLER and symbol rate shown in the Table \ref{table:com} also prove the superiority of the MAB based scheme over the competing schemes.

It can be seen from the results that MAB is able to achieve a tight control of the BLER and this property can be exploited by setting different target BLERs for different users in the system depending upon the traffic type and QoS demanded by each user. Moreover, if the condition to be met is a system-wide target BLER, certain users can be given a higher than target BLER and others a lower than target BLER such that overall the target BLER is achieved. By using ideas from reinforcement learning which are appropriate for the problem at hand, we are able to guarantee a target BLER for each UE while achieving a high throughput. On the other hand the adhoc OLLA schemes such as clustering cannot guarantee a specific BLER because there is no mathematical method behind deriving a cluster length and all results are highly system specific.  
\section{Conclusions and Future Work}
We have been able to map the OLLA problem in cellular networks to an MAB problem and successfully solve it. Hence we use a mathematical framework of reinforcement learning and its tools for obtaining OLLA algorithms which are actually able to maintain a target BLER with nearly 90\% of the users achieving a BLER less than the target BLER, while simultaneously achieving a high throughput. Bounds from large deviation theory is used to significantly reduce the arm sample complexity. Further reduction is then achieved using confidence bounds. We believe that our result is general and can be used to provide reduced sample complexity for any MAB problem whose target reward is known.  For the majority of the users we outperform the ad-hoc scheme and by strictly adhering to the target BLER the proposed scheme can help us achieve a desired latency and QoS. In the current work, we solve an independent MAB problem for each UE, but keep the same target BLER for each UE. A possible extension is to study the case where the target BLER is made to be a function of its traffic QoS and the average $\sinr$ and study the techniques presented here.  Furthermore, by providing a mathematically rigorous approach we have made the OLLA problem open to further improvements and the MAB based technique can be further exploited.
 \bibliography{learning_REF.bib}
\bibliographystyle{IEEEtran}
\end{document}